\documentclass[12pt]{article}

\usepackage{hyperref}
\usepackage{color}
\usepackage[T1]{fontenc}    
\usepackage{url}            
\usepackage{booktabs}       
\usepackage{amsfonts}       
\usepackage{nicefrac}       
\usepackage{microtype}      
\usepackage{microtype}
\usepackage{graphicx}
\usepackage{subfigure}
\usepackage{booktabs} 
\usepackage{amsfonts}
\usepackage{bbm}
\usepackage{amssymb,amsmath,amsthm}
\usepackage{amsmath}
\usepackage{amsfonts}
\usepackage{mathrsfs}
\usepackage{amssymb}
\usepackage{amsthm}
\usepackage{wrapfig}
\newtheorem{definition}{Definition}
\newtheorem{proposition}{Proposition}
\usepackage{thmtools}
\usepackage{thm-restate}
\usepackage{hyperref}

\usepackage{cleveref}

\usepackage{graphicx,subfigure}
\usepackage{tikz}
\usepackage{comment}

\usepackage{epstopdf}

\usepackage{libertine}

\usepackage{algorithm}
\usepackage{algorithmic}

\usepackage[numbers]{natbib}
\newdimen\arrowsize
\pgfarrowsdeclare{arcsq}{arcsq}
{
  \arrowsize=0.2pt
  \advance\arrowsize by .5\pgflinewidth
  \pgfarrowsleftextend{-4\arrowsize-.5\pgflinewidth}
  \pgfarrowsrightextend{.5\pgflinewidth}
}
{
  \arrowsize=1.5pt
  \advance\arrowsize by .5\pgflinewidth
  \pgfsetdash{}{0pt} 
  \pgfsetroundjoin   
  \pgfsetroundcap    
  \pgfpathmoveto{\pgfpoint{0\arrowsize}{0\arrowsize}}
  \pgfpatharc{-90}{-140}{4\arrowsize}
  \pgfusepathqstroke
  \pgfpathmoveto{\pgfpointorigin}
  \pgfpatharc{90}{140}{4\arrowsize}
  \pgfusepathqstroke
}

\newtheorem{corollary}{Corollary}

\newtheorem{assumption}{Assumption}
\usepackage{chngcntr}
\usepackage{apptools}
\AtAppendix{\counterwithin{theorem}{section}}
\newtheorem{theorem}{Theorem}
\AtAppendix{\counterwithin{lemma}{section}}
\newtheorem{lemma}{Lemma}
\AtAppendix{\counterwithin{equation}{section}}

\title{Theoretical Analysis of Adversarial Learning:\\ A Minimax Approach}

\font\myfont=cmr12 at 13pt

\author{\myfont Zhuozhuo~Tu\thanks{UBTECH Sydney AI Centre and the School of Computer Science in the Faculty of Engineering and Information Technologies at The University of Sydney, NSW, 2006, Australia, zhtu3055@uni.sydney.edu.au, zjin8228@uni.sydney.edu.au, dacheng.tao@sydney.edu.au.} \ \ \   Jingwei~Zhang\footnotemark[1] \ \ \    Dacheng~Tao\footnotemark[1]}

\date{}

\begin{document}

\newcommand*\sqcitep[1]{{\setcitestyle{square}$\!\!$\citep{#1}}}

\maketitle

\abstract{
Here we propose a general theoretical method for analyzing the risk bound in the presence of adversaries. Specifically, we try to fit the adversarial learning problem into the minimax framework. We first show that the original adversarial learning problem can be reduced to a minimax statistical learning problem by introducing a transport map between distributions. Then, we prove a new risk bound for this minimax problem in terms of covering numbers under a weak version of Lipschitz condition. Our method can be applied to multi-class classification problems and commonly used loss functions such as the hinge and ramp losses. As some illustrative examples, we derive the adversarial risk bounds for SVMs, deep neural networks, and PCA, and our bounds have two data-dependent terms, which can be optimized for achieving adversarial robustness.
}
\newpage
\section{Introduction}
Machine learning models, especially deep neural networks, have achieved impressive performance across a variety of domains including image classification, natural language processing, and speech recognition. However, these techniques can easily be fooled by adversarial examples, i.e., carefully perturbed input samples aimed to cause misclassification during the test phase. This phenomenon was first studied in spam filtering \cite{dal,low,low2} and has attracted considerable attention since 2014, when \citet{sze} noticed that small perturbations in images can cause misclassification in neural network classifiers. Since then, there has been considerable focus on developing adversarial attacks against machine learning algorithms \cite{good,big,big1, athalye2018obfuscated, uesato2018adversarial}, and, in response, many defense mechanisms have also been proposed to counter these attacks \cite{gu,glob,dek, suggala2018adversarial, madry2018towards}. These works focus on creating optimization-based robust algorithms, but their generalization performance under adversarial input perturbations is still not fully understood.

\citet{sch} recently discussed the generalization problem in the adversarial setting and showed that the sample complexity of learning a specific distribution in the presence of $l_\infty$-bounded adversaries increases by an order of $\sqrt{d}$ for all classifiers. The same paper recognized that deriving the agnostic-distribution generalization bound remained an open problem \cite{sch}. In a subsequent study, \citet{cul} extended the standard PAC-learning framework to the adversarial setting by defining a corrupted hypothesis class and showed that the VC dimension of this corrupted hypothesis class for halfspace classifiers does not increase in the presence of an adversary. While their work provided a theoretical understanding of the problem of learning with adversaries, it had two limitations. First, their results could only be applied to binary problems, whereas in practice we usually need to handle multi-class problems. Second, the 0-1 loss function used in their work is not convex and thus very hard to optimize. 

In this paper, we propose a general theoretical method for analyzing generalization performance in the presence of adversaries. In particular, we attempt to fit the adversarial learning problem into the minimax framework \cite{lee}. In contrast to traditional statistical learning, where the underlying data distribution $P$ is unknown but fixed, the minimax framework considers the uncertainty about the distribution $P$ by introducing an ambiguity set and then aims to minimize the risk with respect to the worst-case distribution in this set. Motivated by \citet{lee}, we first note that the adversarial expected risk over a distribution $P$ is equivalent to the standard expected risk under a new distribution $P'$. Since this new distribution is not fixed and depends on the hypothesis, we instead consider the worst case. In this way, the original adversarial learning problem is reduced to a minimax problem, and we use the minimax approach to derive the risk bound for the adversarial expected risk. Our contributions can be summarized as follows.
\begin{itemize}
\item We propose a general method for analyzing the risk bound in the presence of adversaries. Our method is general in several respects. First, the adversary we consider is general and encompasses all $l_p$ bounded adversaries for $p\geq 1$. Second, our method can be applied to multi-class problems and other commonly used loss functions such as the hinge loss and ramp loss, whereas \citet{cul} only considered the binary classification problem and the 0-1 loss.
\item We prove a new bound for the local worst-case risk under a weak version of Lipschitz condition. Our bound is always better than that of \citet{lee2018minimax}, and can recover the standard non-adversarial risk bound by setting the radius $\epsilon_\mathcal{B}$ of the adversary to 0, whereas \citet{lee2018minimax} give a $\epsilon_\mathcal{B}$-free bound.
\item We derive the adversarial risk bounds for SVM, deep neural networks, and PCA. Our bounds have two data-dependent terms, suggesting that minimizing the sum of the two terms can help achieve adversarial robustness.
\end{itemize}

The remainder of this paper is structured as follows. In Section~\ref{sec2}, we discuss related works. Section~\ref{sec3} formally defines the problem, and we present our theoretical method in Section~\ref{sec4}. The adversarial risk bounds for SVM, neural networks, and PCA are described in Section~\ref{sec5}, and we conclude and discuss future directions in Section~\ref{sec6}.

\section{Related Work}
\label{sec2}
Our work leverages some of the benefits of statistical machine learning, summarized as follows.
\subsection{Generalization in Supervised Learning}
Generalization is a central problem in supervised learning, and the generalization capability of learning algorithms has been extensively studied. Here we review the salient aspects of generalization in supervised learning relevant to this work.

Two main approaches are used to analyze the generalization bound of a learning algorithm. The first is based on the complexity of the hypothesis class, such as the VC dimension \cite{vapnik2013nature,vapnik1999overview} for binary classification, Rademacher and Gaussian complexities \cite{bartlett2002rademacher, bartlett2005local},  and the covering number \cite{zhou2002covering, zhang2002covering,bartlett2017spectrally}. Note that hypothesis complexity-based analyses of generalization error are algorithm independent and consider the worst-case generalization over all functions in the hypothesis class. In contrast, the second approach is based on the properties of a learning algorithm and is therefore algorithm dependent. The properties characterizing the generalization of a learning algorithm include, for example, algorithmic stability \cite{bousquet2002stability, shalev2010learnability, liu2017algorithmic}, robustness \cite{xu2012robustness}, and algorithmic luckiness \cite{herbrich2002algorithmic}. Some other methods exist for analyzing the generalization error in machine learning such as the PAC-Bayesian approach \cite{neyshabur2017pac,ambroladze2007tighter}, compression-based bounds \cite{langford2005tutorial, arora2018stronger}, and information-theoretic approaches \cite{xu2017information, alabdulmohsin2015algorithmic, pmlr-v51-russo16, 2018arXiv180409060Z}.

\subsection{Minimax Statistical Learning}
In contrast to standard empirical risk minimization in supervised learning, where test data follow the same distribution as training data, minimax statistical learning arises in problems of distributionally robust learning \cite{far, gao, lee, lee2018minimax, sinha2018certifying} and minimizes the worst-case risk over a family of probability distributions. Thus, it can be applied to the learning setting in which the test data distribution differs from that of the training data, such as in domain adaptation and transfer learning \cite{courty2017optimal}.  In particular, \citet{gao} proposed a dual representation of worst-case risk over the ambiguity set of probability distributions, which was given by balls in Wasserstein space. Then, \citet{lee} derived the risk bound for minimax learning by exploiting the dual representation of worst-case risk proposed by \citet{gao}. However, the minimax risk bound proposed in \citet{lee} would go to infinity and thus become vacuous as $\epsilon_\mathcal{B}\rightarrow 0$. During the preparation of the initial draft of this paper, \citet{lee2018minimax} presented a new bound by imposing a Lipschitz assumption to avoid this problem. However, their new bound was $\epsilon_\mathcal{B}$-free and cannot recover the usual risk bound by setting $\epsilon_\mathcal{B}=0$. \citet{sinha2018certifying} also provided a similar upper bound on the worst-case population loss over distributions defined by a certain distributional Wasserstein distance, and their bound was efficiently computable by a principled adversarial training procedure and hence certified a level of robustness. However their training procedure required that the penalty parameter should be large enough and thus can only achieve a small amount of robustness. Here we improve on the results in \citet{lee,lee2018minimax} and present a new risk bound for the minimax problem.

\subsection{Learning with Adversaries}
The existence of adversaries during the test phase of a learning algorithm makes learning systems untrustworthy. There is extensive literature on analysis of adversarial robustness \cite{wang2017analyzing, fawzi2016robustness, hein2017formal, gilmer2018adversarial} and design of provable defense against adversarial attacks\cite{wong2018provable, raghunathan2018certified, madry2018towards, sinha2018certifying}, in contrast to the relatively limited literature on risk bound analysis of adversarial learning. A comprehensive review of works on adversarial machine learning can be found in \citet{biggio2018wild}. Concurrently to our work,  \citet{khim2018adversarial} and \citet{yin2018rademacher} provided different approaches for deriving adversarial risk bounds. \citet{khim2018adversarial} derived adversarial risk bounds for linear classifiers and neural networks using a method called function transformation. However, their approach can only be applied to binary classification. \citet{yin2018rademacher} gave similar adversarial risk bounds as \citet{khim2018adversarial} through the lens of Rademacher complexity. Although they provided risk bounds in multi-class setting, their work focused on $l_\infty$ adversarial attacks and was limited to one-hidden layer ReLU neural networks. After the initial preprint of this paper, \citet{khim2018badversarial} extended their method to multi-class setting at the expense of incurring an extra factor of the number of classes in their bound. In contrast, our multi-class bound does not have explicit dependence on this number. We hope that our method can provide new insight into analysis of the adversarial risk bounds.

\section{Problem Setup}\label{sec3}
We consider a standard statistical learning framework. Let $\mathcal{Z}=\mathcal{X}\times\mathcal{Y}$ be a measurable instance space where $\mathcal{X}$ and $\mathcal{Y}$ represent feature and label spaces, respectively. We assume that examples are independently and identically distributed according to some fixed but unknown distribution $P$. The learning problem is then formulated as follows. The learner considers a class $\mathcal{H}$ of hypothesis $h:\mathcal{X}\rightarrow \mathcal{Y}$ and a loss function $l: \mathcal{Y}\times \mathcal{Y}\rightarrow \mathbb{R}_{+}$. The learner receives $n$ training examples denoted by $S=((x_1,y_1),(x_2,y_2),\cdots,(x_n,y_n))$ drawn i.i.d. from $P$ and tries to select a hypothesis $h\in\mathcal{H}$ that has a small expected risk. However, in the presence of adversaries, there will be imperceptible perturbations to the input of examples, which are called adversarial examples. We assume that the adversarial examples are generated by adversarially choosing an example from neighborhood $N(x)$. We require $N(x)$ to be nonempty and that some choice of examples is always available. Throughout this paper, we assume that $N(x)=\{x':x'-x\in \mathcal{B}\}$, where $\mathcal{B}$ is a nonempty, closed, convex, origin-symmetric set. Note that the definition of $\mathcal{B}$ is very general and encompasses all $l_p$ -bounded adversaries when $p\geq1$. We next give the formal definition of adversarial expected and empirical risk to measure the learner's performance in the presence of adversaries. 
\begin{definition}
(Adversarial Expected Risk). The adversarial expected risk of a hypothesis $h\in\mathcal{H}$ over the distribution $P$ in the presence of an adversary constrained by $\mathcal{B}$ is
$$R_P(h,\mathcal{B})=\mathbb{E}_{(x,y)\sim P}[\max_{x'\in N(x)}l(h(x'),y)].$$
\end{definition}
If $\mathcal{B}$ is the zero-dimensional space $\{\bold{0}\}$, then the adversarial expected risk will reduce to the standard expected risk without an adversary. 
Since the true distribution is usually unknown, we instead use the empirical distribution to approximate the true distribution, which is equal to $(x_i,y_i)$ with probability $1/n$ for each $i\in\{1,\cdots, n\}$. That gives us the following definition of adversarial empirical risk.
\begin{definition}
(Adversarial Empirical Risk ). The adversarial empirical risk of $h$ in the presence of an adversary constrained by $\mathcal{B}$ is
$$R_{P_n}(h,\mathcal{B})=\mathbb{E}_{(x,y)\sim P_n}[\max_{x'\in N(x)}l(h(x'),y)],$$
where $P_n$ represents the empirical distribution.
\end{definition}
In the next section, we derive the adversarial risk bounds.
\section{Main Results}\label{sec4}
In this section, we present our main results. The trick is to pushforward the original distribution $P$ into a new distribution $P'$ using a transport map $T_h:\mathcal{Z}\rightarrow\mathcal{Z}$ satisfying
$$R_P(h,\mathcal{B})=R_{P'}(h),$$ where $R_{P'}(h)=\mathbb{E}_{(x,y)\sim P'}l(h(x),y)$ is the standard expected risk without the adversary. Therefore, an upper bound on the expected risk over the new distribution leads to an upper bound on the adversarial expected risk. 

Note that the new distribution $P'$ is not fixed and depends on the hypothesis $h$. As a result, traditional statistical learning cannot be directly applied. However, note that these new distributions lie within a Wasserstein ball centered on $P$. If we consider the worst case within this Wasserstein ball, then the original adversarial learning problem can be reduced to a minimax problem. We can therefore use the minimax approach to derive the adversarial risk bound. We first introduce the Wasserstein distance and minimax framework.
\subsection{Wasserstein Distance and Local Worst-case Risk}\label{subs}
Let $(\mathcal{Z},d_\mathcal{Z})$ be a metric space where $\mathcal{Z}=\mathcal{X}\times\mathcal{Y}$ and $d_{\mathcal{Z}}$ is defined as
$$d_{\mathcal{Z}}^p(z,z')=d_{\mathcal{Z}}^p((x,y),(x',y'))=(d_{\mathcal{X}}^p(x,x')+d_{\mathcal{Y}}^p(y,y'))$$
with $d_{\mathcal{X}}$ and $d_{\mathcal{Y}}$ representing the metric in the feature space and label space respectively. For example, if $\mathcal{Y}=\{1,-1\}$, $d_{\mathcal{Y}}(y,y')$ can be $\mathbbm{1}_{(y\neq y')}$, and if $\mathcal{Y}=[-B,B]$, $d_{\mathcal{Y}}(y,y')$ can be $(y-y')^2$. In this paper, we require that $d_\mathcal{X}$ is translation invariant, i.e., $d_\mathcal{X}(x,x')=d_\mathcal{X}(x-x',0)$. With this metric, we denote with $\mathcal{P}(\mathcal{Z})$ the space of all Borel probability measures on $\mathcal{Z}$, and with $\mathcal{P}_p(\mathcal{Z})$ the space of all $P\in \mathcal{P}(\mathcal{Z})$ with finite $p$th moments for $p\geq 1$:
$$\mathcal{P}_p(\mathcal{Z}):=\{P\in \mathcal{P}(\mathcal{Z}):\mathbb{E}_{P}[d_{\mathcal{Z}}^p(z,z_0)]< \infty \ for\ z_0\in \mathcal{Z}\}.$$
Then, the $p$th Wasserstein distance between two probability measures $P, Q\in \mathcal{P}_p(\mathcal{Z})$ is defined as

$$W_p(P,Q):=\inf_{M\in\Gamma(P,Q)}(\mathbb{E}_M[d_{\mathcal{Z}}^p(z,z')])^{1/p},$$
where $\Gamma(P,Q)$ denotes the collection of all measures on $\mathcal{Z}\times \mathcal{Z}$ with marginals P and Q on the first and second factors, respectively.

Now we define the local worst-case risk of $h$ at $P$,
$$R_{\epsilon,p}(P,h):=\sup_{Q\in B_{\epsilon,p}^W(P)}R_Q(h),$$
where $B_{\epsilon,p}^W(P):=\{Q\in \mathcal{P}_p(Z): W_p(P,Q))\leq \epsilon\}$ is the $p$-Wasserstein ball of radius $\epsilon\geq0$ centered at $P$. 

With these definitions, we next show the adversarial expected risk can be related to the local worst-case risk by a transport map $T_h$.
\subsection{Transport Map}
Define a mapping $T_h:\mathcal{Z}\rightarrow\mathcal{Z}$
$$z=(x,y)\rightarrow(x^*,y),$$
where $x^*=\arg\max_{x'\in N(x)}l(h(x'),y)$. By the definition of $d_\mathcal{Z}$, it is easy to obtain $d_\mathcal{Z}((x,y),(x^*,y))=d_\mathcal{X}(x,x^*)$. We now prove that the adversarial expected risk can be related to the standard expected risk via the mapping $T_h$.
\begin{lemma}
Let $P'=T_h\#P$, the pushforward of $P$ by $T_h$, then we have 
$$R_P(h,\mathcal{B})=R_{P'}(h).$$
\end{lemma}
\begin{proof}
By the definition, we have
\begin{equation*} 
\begin{array}{ll}
R_P(h,\mathcal{B})=\mathbb{E}_{(x,y)\sim P}[\max_{x'\in N(x)}l(h(x'),y)]\\
=\mathbb{E}_{(x,y)\sim P}[l(h(x^*),y)]=\mathbb{E}_{(x,y)\sim P'}[l(h(x),y)]\\
\end{array}
\end{equation*}
So $R_P(h,\mathcal{B})=R_{P'}(h)$.
\end{proof}

By this lemma, the adversarial expected risk over a distribution $P$ is equivalent to the standard expected risk over a new distribution $P'$. However since the new distribution is not fixed and depends on the hypothesis $h$, traditional statistical learning cannot be directly applied. Luckily, the following lemma proves that all these new distributions locate within a Wasserstein ball centered at $P$.
\begin{lemma}
Define the radius of the adversary $\mathcal{B}$ as $\epsilon_\mathcal{B}:=\sup_{x\in\mathcal{B}}d_{\mathcal{X}}(x,0)$. For any hypothesis $h$ and the corresponding $P'=T_h\#P$, we have
$$W_p(P,P')\leq \epsilon_\mathcal{B}.$$
\end{lemma}
\begin{proof}
By the definition of Wasserstein distance, 
$$W_p^p(P,P')\leq \mathbb{E}_P[d^p_{\mathcal{Z}}(Z,T_h(Z))]=\mathbb{E}_P[d^p_{\mathcal{\mathcal{X}}}(x,x^*)]\leq \epsilon_\mathcal{B}^p,$$
where the last inequality uses the translation invariant property of $d_\mathcal{X}$. Therefore, we have $W_p(P,P')\leq \epsilon_\mathcal{B}$.
\end{proof}

From this lemma, we can see that all possible new distributions lie within a Wasserstein ball of radius $\epsilon_\mathcal{B}$ centered on $P$. So, by upper bounding the worst-case risk in the ball, we can bound the adversarial expected risk. The relationship between local worst-case risk and adversarial expected risk is as follows. Note that this inequality holds for any $p\geq1$. So, in the rest of the paper, we only discuss the case $p=1$; that is,
\begin{equation}\label{e1}
R_{P}(h, \mathcal{B})\leq R_{\epsilon_\mathcal{B},1}(P,h), \quad \forall h\in \mathcal{H}.
\end{equation}
\subsection{Adversarial Risk Bounds}
In this subsection, we first prove a bound for the local worst-case risk. Then, the adversarial risk bounds can be derived directly by (\ref{e1}). For the convenience of our discussion, we denote a function class $\mathcal{F}$ by compositing the functions in $\mathcal{H}$ with the loss function $l(\cdot,\cdot)$, i.e., $\mathcal{F}=\{(x,y)\rightarrow l(h(x),y): h\in \mathcal{H}\}$. The key ingredient of a bound on the local worst-case risk is the following strong duality result after \citet{gao}:
\begin{proposition}
For any upper semicontinuous function $f:\mathcal{Z}\rightarrow\mathbb{R}$ and for any $P\in \mathcal{P}_p(\mathcal{Z})$,
$$R_{\epsilon_\mathcal{B},1}(P,f)=\min_{\lambda\geq0}\{\lambda \epsilon_\mathcal{B}+\mathbb{E}_P[\varphi_{\lambda,f}(z)]\},$$
where $\varphi_{\lambda,f}(z):=\sup_{z'\in\mathcal{Z}}\{f(z')-\lambda\cdot d_\mathcal{Z}(z,z')\}$.
\end{proposition}

We begin with some assumptions.
\begin{assumption}\label{ass1}
The instance space $\mathcal{Z}$ is bounded: $diam(\mathcal{Z}):=\sup_{z,z'\in \mathcal{Z}}d_\mathcal{Z}(z,z')< \infty$.
\end{assumption}
\begin{assumption}\label{ass2}
The functions in $\mathcal{F}$ are upper semicontinuous and uniformly bounded: $0\leq f(z)\leq M<\infty$ for all $f\in\mathcal{F}$ and $z\in \mathcal{Z}$.
\end{assumption}
\begin{assumption}\label{ass3}
For any function $f\in\mathcal{F}$ and any $z\in\mathcal{Z}$, there exists a constant $\lambda$ such that 
$f(z')-f(z)\leq\lambda d_{\mathcal{Z}}(z,z')$ for any $z'\in\mathcal{Z}$.
\end{assumption}
Note that Assumption 3 is a weak version of Lipschitz condition since the constant $\lambda$ is not fixed and depends on $f$ and $z$. It is easy to see that if the function $f\in\mathcal{F}$ is $L$-Lipschitz with respect to the metric $d_\mathcal{Z}$, i.e., $|f(z)-f(z')|\leq Ld_\mathcal{Z}(z,z')$, Assumption 3 automatically holds with $\lambda$ always being $L$. Assumption 3 is very straightforward. But it is not easy to use for our proof. For this sake, we give an equivalent expression to Assumption 3 in the following lemma.
\begin{lemma}
Assumption 3 holds if and only if for any function $f\in\mathcal{F}$ and any empirical distribution $P_n$, the set $\{\lambda: \psi_{f,P_n}(\lambda)=0\}$ is nonempty, where $\psi_{f,P_n}(\lambda):=\mathbb{E}_{P_n}(\sup_{z'\in \mathcal{Z}}\{f(z')-\lambda d_{\mathcal{Z}}(z,z')-f(z)\})$. 
\end{lemma}
The proof of this lemma is contained in Appendix~\ref{app10}.

We denote the smallest value in the set as $\lambda_{f,P_n}^+:=\inf\{\lambda:  \psi_{f,P_n}(\lambda)=0\}$.
In order to prove the local worst-case risk bound, we need two technical lemmas.
\begin{lemma}
Fix some $f\in\mathcal{F}$. Define $\bar{\lambda}$ via 
$$\bar{\lambda}:=\arg\min_{\lambda\geq0}\{\lambda \epsilon_\mathcal{B}+\mathbb{E}_{P_n}[\varphi_{\lambda,f}(Z)]\}.$$
Then 
\begin{equation}\label{eq90}
\bar{\lambda}\in \left\{
\begin{array}{ll}
[0,\dfrac{M}{\epsilon_\mathcal{B}}] &if \ \epsilon_\mathcal{B}\geq \dfrac{M}{\lambda_{f,P_n}^+}\\ \relax
[\lambda_{f,P_n}^{-},\lambda_{f,P_n}^+] & if\  \epsilon_\mathcal{B}< \dfrac{M}{\lambda_{f,P_n}^+} \\
\end{array}
\right.,
\end{equation}
where $\lambda_{f,P_n}^-:=\sup\{\lambda: \psi_{f,P_n}(\lambda)=\lambda_{f,P_n}^+\cdot \epsilon_\mathcal{B}\}$ if the set $\{\lambda: \psi_{f,P_n}(\lambda)=\lambda_{f,P_n}^+\cdot \epsilon_\mathcal{B}\}$ is nonempty, otherwise $\lambda_{f,P_n}^-:=0$.
\end{lemma}
\begin{proof}
If $\epsilon_\mathcal{B}\geq \dfrac{M}{\lambda_{f,P_n}^+}$, by Proposition 1, $R_{\epsilon_\mathcal{B},1}(P_n,f)=\bar{\lambda} \epsilon_\mathcal{B}+\mathbb{E}_{P_n}[\varphi_{\bar{\lambda},f}(Z)]$ , we have
$$\bar{\lambda} \epsilon_\mathcal{B}\leq R_{\epsilon_\mathcal{B},1}(P_n,f).$$
Since $f(z)\leq M$ for any $z$, we get $R_{\epsilon_\mathcal{B},1}(P_n,f)\leq M$. So $\bar{\lambda}\leq \dfrac{M}{\epsilon_\mathcal{B}}.$

For the other side, we first show that $\psi_{f,P_n}(\lambda)$ is continuous and monotonically non-increasing. The monotonicity is easy to verify from the definition. For continuity, for any $\lambda_2> \lambda_1$, suppose that 
\begin{equation*}
\begin{array}{l}
\hat{z}=\sup_{z'\in \mathcal{Z}}\{f(z')-\lambda_1 d_{\mathcal{Z}}(z,z')-f(z)\}),\\
z^*=\sup_{z'\in \mathcal{Z}}\{f(z')-\lambda_2 d_{\mathcal{Z}}(z,z')-f(z)\}).
\end{array}
\end{equation*}
Then we have
\begin{equation*}
\begin{array}{l}
\psi_{f,P_n}(\lambda_1)-\psi_{f,P_n}(\lambda_2)\\
=\mathbb{E}_{P_n}(\sup_{z'\in \mathcal{Z}}\{f(z')-\lambda_1 d_{\mathcal{Z}}(z,z')-f(z)\}-\\
\sup_{z'\in \mathcal{Z}}\{f(z')-\lambda_2 d_{\mathcal{Z}}(z,z')-f(z)\})\\
\leq \mathbb{E}_{p_n}\left((\lambda_2-\lambda_1)d_\mathcal{Z}(z,\hat{z})\right) \leq (\lambda_2-\lambda_1) diam(\mathcal{Z})\\
\end{array}.
\end{equation*}
So $\psi_{f,P_n}(\lambda)$ is $diam(\mathcal{Z})$-Lipschitz and thus continuous. 

Now we prove $\bar{\lambda}\in [\lambda_{f,P_n}^-,\lambda_{f,P_n}^+]$. If $\lambda>\lambda_{f,P_n}^+$, by the monotonicity and nonnegativity of $\psi_{f,P_n}(\lambda)$, we have $\psi_{f,P_n}(\lambda)=\psi_{f,P_n}(\lambda_{f,P_n}^+)=0$, which implies $\lambda \epsilon_\mathcal{B}+\mathbb{E}_{P_n}[\varphi_{\lambda,f}(z)]\geq\lambda_{f,P_n}^+\epsilon_\mathcal{B}+\mathbb{E}_{P_n}[\varphi_{\lambda_{f,P_n}^+,f}(z)]$. Therefore the optimal $\bar{\lambda}\leq \lambda_{f,P_n}^+$. To show $\bar{\lambda}\geq \lambda_{f,P_n}^-$, first notice that $\psi_{f,P_n}(\lambda)$ belongs to $[0, M]$ for any $\lambda$. We define $$\lambda_{f,P_n}^-:=\sup\{\lambda: \psi_{f,P_n}(\lambda)=\lambda_{f,P_n}^+\cdot \epsilon_\mathcal{B}\}.$$
Note that this set $\{\lambda: \psi_{f,P_n}(\lambda)=\lambda_{f,P_n}^+\cdot \epsilon_\mathcal{B}\}$ might be empty if $\psi_{f,P_n}(0)<\lambda_{f,P_n}^+\cdot \epsilon_\mathcal{B}<M$. In this case, we just let $\lambda_{f,P_n}^-=0$, and $\bar{\lambda}$ must belong to $[0,\lambda_{f,P_n}^+]$. Otherwise, there must exist some $\lambda\in [0,\lambda_{f,P_n}^+]$ which satisfies $\psi_{f,P_n}(\lambda)= \lambda_{f,P_n}^+\cdot \epsilon_\mathcal{B}$ by the intermediate value theorem of a continuous function. We choose $\lambda_{f,P_n}^-$ to be the maximal one in that set. Then, for any $\lambda< \lambda_{f,P_n}^-$, since $\psi_{f,P_n}(\lambda)$ is monotonically non-increasing, we have
 $$\mathbb{E}_{P_n}(\sup_{z'\in \mathcal{Z}}\{f(z')-\lambda d_{\mathcal{Z}}(z,z')-f(z)\})\geq \lambda_{f,P_n}^+\cdot \epsilon_\mathcal{B}.$$
 By rearranging the items on both sides, we obtain
$$ \lambda \epsilon_\mathcal{B}+\mathbb{E}_{P_n}[\varphi_{\lambda,f}(z)]\geq  \lambda_{f,P_n}^+\cdot \epsilon_\mathcal{B}+\mathbb{E}_{P_n}(f(z))$$
for any $\lambda< \lambda_{f,P_n}^-$. Therefore, $\bar{\lambda}\geq \lambda_{f,P_n}^-$, and we complete the proof.
\end{proof}
\noindent
\textbf{Remark 1.} We can show $\lim_{\epsilon_\mathcal{B}\rightarrow 0}\lambda_{f,P_n}^{-}=\lambda_{f,P_n}^{+}$ by using $(\epsilon,\delta)$ language as follows. $\forall \epsilon$, define $\delta=\frac{\psi_{f,P_n}(\lambda_{f,P_n}^+-\epsilon)}{\lambda_{f,P_n}^+}$. Then, for any $\epsilon_\mathcal{B}<\delta$, we have $\psi_{f,P_n}(\lambda_{f,P_n}^+-\epsilon)>\lambda_{f,P_n}^+\cdot \epsilon_\mathcal{B}$. By the definition of $\lambda_{f,P_n}^{-}$, $\psi_{f,P_n}(\lambda_{f,P_n}^{-})=\lambda_{f,P_n}^+\cdot \epsilon_\mathcal{B}$. Since $\psi_{f,P_n}(\lambda)$ is monotonically non-increasing, we have $\lambda_{f,P_n}^{-}> \lambda_{f,P_n}^{+}-\epsilon$. Therefore, $\lim_{\epsilon_\mathcal{B}\rightarrow 0}\lambda_{f,P_n}^{-}=\lambda_{f,P_n}^{+}$.
\begin{lemma}
Define the function class $\Phi:=\{\varphi_{\lambda,f}: \lambda\in [a,b], f\in\mathcal{F}\}$ where $b\geq a\geq 0$. Then, the expected Rademacher complexity of the function class $\Phi$ satisfies
\begin{equation*}
\begin{array}{lr}
\mathfrak{R}_n(\Phi)\leq&\dfrac{12\mathfrak{C}(\mathcal{F})}{\sqrt{n}}+
\dfrac{6\sqrt{\pi}}{\sqrt{n}}(b-a)\cdot diam(Z),
\end{array}
\end{equation*}
where $\mathfrak{C}(\mathcal{F}):=\int_0^\infty\sqrt{log\mathcal{N}(\mathcal{F},||\cdot||_\infty,u/2)}du$ and $\mathcal{N}(\mathcal{F},||\cdot||_\infty,u/2)$ denotes the covering number of $\mathcal{F}$.
\end{lemma}

The proof of this lemma is contained in Appendix~\ref{app1}.

The following theorem gives the generalization bound for the local worst-case risk. We first introduce the corresponding notation: $\bar{\lambda}\in[\zeta^-_{f,P_n},\zeta^+_{f,P_n}]$ denotes expression (\ref{eq90}), $[\zeta^-,\zeta^+]:=\bigcup_{f, P_n}[\zeta^-_{f,P_n},\zeta^+_{f,P_n}]$ and $\Lambda_{\epsilon_\mathcal{B}}:=\zeta^+-\zeta^-$. It is straightforward to check that $[\zeta^-,\zeta^+]\subset [0,M/\epsilon_\mathcal{B}]$ from expression (\ref{eq90}).
\begin{theorem}
If the assumptions~\ref{ass1}-~\ref{ass3} hold, then for any $f\in \mathcal{F}$, we have
\begin{equation*} 
\begin{array}{l}
R_{\epsilon_\mathcal{B},1}(P,f)-R_{\epsilon_\mathcal{B},1}(P_n,f)\leq \dfrac{24\mathfrak{C}(\mathcal{F})}{\sqrt{n}}+M\sqrt{\dfrac{log(\frac{1}{\delta})}{2n}}+\\
\ \ \ \ \ \ \ \ \ \ \ \ \ \ \ \ \ \ \ \ \ \ \ \ \ \ \ \ \ \ \ \ \ \ \ \ \ \ \ \ \ \ \ \ \ \ \ \ \ \ \ \dfrac{12\sqrt{\pi}}{\sqrt{n}}\Lambda_{\epsilon_\mathcal{B}} \cdot diam(Z)
\end{array}
\end{equation*}
with probability at least $1-\delta$.
\end{theorem}
\begin{proof}
For any $f\in\mathcal{F}$, define 
$$\bar{\lambda}:=\arg{\min_{\lambda\geq0}\{\lambda \epsilon_\mathcal{B}+\mathbb{E}_{P_n}[\varphi_{\lambda,f}(Z)]\}}.$$
Then using Proposition 1, we can write
\begin{equation*}
\begin{array}{l}
R_{\epsilon_\mathcal{B},1}(P,f)-R_{\epsilon_\mathcal{B},1}(P_n,f)\\
=\min_{\lambda\geq0}\left\{\lambda \epsilon_\mathcal{B}+\displaystyle\int_{\mathcal{Z}}\varphi_{\lambda,f}(z)P(dz)\right\}-\bigg(\bar{\lambda}\epsilon_\mathcal{B}+\\
\left.\displaystyle\int_{\mathcal{Z}}\varphi_{\bar{\lambda},f}(z)P_n(dz)\right)\leq \displaystyle\int_{\mathcal{Z}}\varphi_{\bar{\lambda},f}(z)(P-P_n)(dz).
\end{array}
\end{equation*}
By lemma 4, we have $\bar{\lambda}\in[\zeta^-_{f,P_n},\zeta^+_{f,P_n}]$. Define the function class $\Phi:=\{\varphi_{\lambda,f}: \lambda\in [\zeta^-,\zeta^+], f\in\mathcal{F}\}$. Then, we have
$$R_{\epsilon_\mathcal{B},1}(P,f)-R_{\epsilon_\mathcal{B},1}(P_n,f)\leq \sup_{\varphi\in\Phi}\left[\int_\mathcal{Z}\varphi(z)(P-P_n)(dz)\right].$$
Since all $f\in\mathcal{F}$ takes values in $[0,M]$, the same holds for all $\varphi\in \Phi$. Therefore, by a standard symmetrization argument \cite{moh},
$$R_{\epsilon_\mathcal{B},1}(P,f)-R_{\epsilon_\mathcal{B},1}(P_n,f)\leq 2\mathfrak{R}_n(\Phi)+M\sqrt{\dfrac{log(1/\delta)}{2n}}$$
with probability at least $1-\delta$, where
$\mathfrak{R}_n(\Phi):=\mathbb{E}[\sup_{\varphi\in\Phi}\dfrac{1}{n}\sum_{i=1}^n\sigma_i\varphi(z_i)]$
is the expected Rademacher complexity of $\Phi$. Using the bound of lemma 4, we get the following result
\begin{equation*}
\begin{array}{l}
R_{\epsilon_\mathcal{B},1}(P,f)-R_{\epsilon_\mathcal{B},1}(P_n,f)\leq \dfrac{24\mathfrak{C}(\mathcal{F})}{\sqrt{n}}+M\sqrt{\dfrac{log(\frac{1}{\delta})}{2n}}+\\
\ \ \ \ \ \ \ \ \ \ \ \ \ \ \ \ \ \ \ \ \ \ \ \ \ \ \ \ \ \ \ \ \ \ \ \ \ \ \ \ \ \ \ \ \ \ \ \ \ \ \ \dfrac{12\sqrt{\pi}}{\sqrt{n}}\Lambda_{\epsilon_\mathcal{B}}\cdot diam(Z)\\
\end{array}.
\end{equation*}
\end{proof}
\noindent
\textbf{Remark 2.} \citet{lee2018minimax} prove a bound with $[\zeta^-,\zeta^+]=[0,L]$ under the Lipschitz assumption with $L$ representing the Lipschitz constant. Our result improves a lot on theirs. First, our Assumption 3 is weaker than the Lipschitz assumption in \citet{lee2018minimax}. Second, even under our weaker assumptions, our bound is still better than that of \citet{lee2018minimax} for the case $\epsilon_\mathcal{B}\geq M/L$ since $[\zeta^-,\zeta^+]\subset [0,M/\epsilon_\mathcal{B}]\subset[0,L]$. Third, if further assuming the same Lipschitz condition as \citet{lee2018minimax}, we can get $[\zeta^-,\zeta^+]\subset [0,L]$ by the definition of $\zeta^-$ and $\zeta^+$, which is always better than the ones in \citet{lee2018minimax}. Finally, the term $\dfrac{12\sqrt{\pi}}{\sqrt{n}}\Lambda_{\epsilon_\mathcal{B}}\cdot diam(Z)$ in our bound will vanish as $\epsilon_\mathcal{B}\rightarrow \infty$ or $\epsilon_\mathcal{B}=0$ whereas \citet{lee2018minimax} give a $\epsilon_\mathcal{B}$-free bound with $\Lambda_{\epsilon_\mathcal{B}}$ always being a constant $L$.

This leads to the following upper bound on the adversarial expected risk.
\begin{corollary}
With the conditions in \textbf{Theorem 1}, for any $f\in \mathcal{F}$, we have
\begin{equation}\label{eq1}
\begin{array}{l}
R_{P}(f, \mathcal{B})\leq \dfrac{1}{n}\sum_{i=1}^nf(z_i)+\displaystyle\min_{\lambda\geq0}\{\lambda \epsilon_\mathcal{B}+\psi_{f,P_n}(\lambda)\}+\\
\ \ \ \ \ \ \ \ \ \ \ \ \ \ \ \ \dfrac{24\mathfrak{C}(\mathcal{F})}{\sqrt{n}}+\dfrac{12\sqrt{\pi}}{\sqrt{n}}\Lambda_{\epsilon_\mathcal{B}}diam(Z)+M\sqrt{\dfrac{log(\frac{1}{\delta})}{2n}}\\
\end{array}
\end{equation}
and 
\begin{equation}\label{eq2}
\begin{array}{l}
R_{P}(f, \mathcal{B})\leq \dfrac{1}{n}\sum_{i=1}^nf(z_i)+\lambda_{f,P_n}^+\epsilon_\mathcal{B}+\dfrac{24\mathfrak{C}(\mathcal{F})}{\sqrt{n}}+\\
\ \ \ \ \ \ \ \ \ \ \ \ \ \ \ \ \ \ \ \ \ \dfrac{12\sqrt{\pi}}{\sqrt{n}}\Lambda_{\epsilon_\mathcal{B}}\cdot diam(Z)+M\sqrt{\dfrac{log(\frac{1}{\delta})}{2n}}\\
\end{array}
\end{equation}
with probability at least $1-\delta$.
\end{corollary}
\begin{proof}
By Proposition 1, $R_{\epsilon_\mathcal{B},1}(P_n,f)$ can be written as
\begin{equation*}
\begin{array}{l}
R_{\epsilon_\mathcal{B},1}(P_n,f)\\
=\displaystyle\min_{\lambda\geq0}\{\lambda \epsilon_\mathcal{B}+\mathbb{E}_{P_n}[\varphi_{\lambda,f}(z)]\}\\
=\displaystyle\min_{\lambda\geq0}\{\lambda \epsilon_\mathcal{B}+\mathbb{E}_{P_n}[\varphi_{\lambda,f}(z)-f(z)]\}+\mathbb{E}_{P_n}[f(z)]\\
=\displaystyle\min_{\lambda\geq0}\{\lambda \epsilon_\mathcal{B}+\psi_{f,P_n}(\lambda)\}+\dfrac{1}{n}\sum_{i=1}^nf(z_i)
\end{array},
\end{equation*}
where the last equality uses the definition of $\psi_{f,P_n}(\lambda)$. Substituting the above equation into Theorem 1, we get result (\ref{eq1}). To obtain (\ref{eq2}), we can make use of the following inequality
\begin{equation*}
\begin{array}{rl}
\displaystyle\min_{\lambda\geq0}\{\lambda \epsilon_\mathcal{B}+\psi_{f,P_n}(\lambda)\}&\leq \lambda_{f,P_n}^+\epsilon_\mathcal{B}+\psi_{f,P_n}(\lambda_{f,P_n}^+)\\
&=\lambda_{f,P_n}^+\epsilon_\mathcal{B}\\
\end{array},
\end{equation*}
where the equality follows from the definition of $\lambda_{f,P_n}^+$. 
\end{proof}
\noindent
\textbf{Remark 3.} We are interested in how the adversarial risk bounds differ from the case in which the adversary is absent. Plugging $\epsilon_\mathcal{B}=0$ into inequality (\ref{eq2}) yields the usual generalization of the form
\begin{equation*}
\begin{array}{l}
R_{P}(h)\leq \dfrac{1}{n}\sum_{i=1}^nf(z_i)+\dfrac{24\mathfrak{C}(\mathcal{F})}{\sqrt{n}}+M\sqrt{\dfrac{log(1/\delta)}{2n}}.
\end{array}
\end{equation*}
So the effect of an adversary is to introduce an extra complexity term $12\sqrt{\pi}\Lambda_{\epsilon_\mathcal{B}}\cdot diam(Z)/{\sqrt{n}}$ and an additional linear term on $\epsilon_\mathcal{B}$ which contributes to the empirical risk.

\noindent
\textbf{Remark 4.}\label{re} As mentioned in Remark 2, the extra complexity term will decrease as $\epsilon_\mathcal{B}$ gets bigger if $\epsilon_\mathcal{B}\geq M/\lambda_{f,P_n}^+$ , indicating that a stronger adversary might have a negative impact on the hypothesis class complexity. This is intuitive, since different hypotheses might have the same performance in the presence of a strong adversary and, therefore, the hypothesis class complexity will decrease. We emphasize that this phenomenon does not occur in our concurrent work \citet{khim2018adversarial} and \citet{yin2018rademacher}. In both of their work, this term will increase linearly as $\epsilon_\mathcal{B}$ grows.

\noindent
\textbf{Remark 5.} We should point out that $\lambda_{f,P_n}^+$ is data-dependent and might be difficult to compute exactly. Luckily we can upper bound it easily. For example, if $f$ is $L$-Lipschitz, by the definition of $\lambda_{f,P_n}^+$, we have $\lambda_{f,P_n}^+\leq L$. See Section \ref{sec5} for more examples. In particular, if $\psi_{f,P_n}(\lambda)\equiv0$ for any $\lambda\geq0$, we get $\lambda_{f,P_n}^+=0$, and the additional term $\lambda_{f,P_n}^+\epsilon_\mathcal{B}$ in (\ref{eq2}) will disappear.

\section{Example Bounds}\label{sec5}
In this section, we illustrate the application of Corollary 1 to several commonly-used models: SVMs, neural networks, and PCA.
\subsection{Support Vector Machines}
We first start with SVMs. Let $\mathcal{Z}=\mathcal{X}\times \mathcal{Y}$, where the feature space $\mathcal{X}=\{x\in \mathbb{R}^d:||x||_2\leq  r\}$ and the label space $\mathcal{Y}=\{-1,+1\}$. Equip $\mathcal{Z}$ with the Euclidean metric
 $$d_{\mathcal{Z}}(z,z')=d_{\mathcal{Z}}((x,y),(x',y'))=||x-x'||_2+\mathbbm{1}_{(y\neq y')}.$$
Consider the hypothesis space $\mathcal{F}=\{(x,y)\rightarrow\max\{0,1-yh(x)\}: h\in H\}$, where $H=\{x\rightarrow w\cdot x: ||w||_2\leq \Lambda\}$. We can now derive the expected risk bound for SVMs in the presence of an adversary.
\begin{corollary}
For the SVM setting considered above, for any $f\in \mathcal{F}$, with probability at least $1-\delta$,
\begin{equation*}
\begin{array}{l}
R_{P}(f, \mathcal{B})
\leq \dfrac{1}{n}\sum_{i=1}^nf(z_i)+\lambda_{f,P_n}^+\epsilon_\mathcal{B}+\dfrac{144}{\sqrt{n}}\Lambda r\sqrt{d}+\\
\ \ \ \ \ \ \ \ \ \ \ \ \ \ \ \ \ \ \ \ \dfrac{12\sqrt{\pi}}{\sqrt{n}}\Lambda_{\epsilon_\mathcal{B}}\cdot (2r+1)+(1+\Lambda r)\sqrt{\dfrac{log(\frac{1}{\delta})}{2n}},
\end{array}
\end{equation*}
where $\lambda_{f,P_n}^+\leq \displaystyle\max_{i}\{2y_iw\cdot x_i, ||w||_2\}$.
\end{corollary}
The proof of Corollary 2 can be found in Appendix~\ref{app3}.

Our result can easily be extended to kernel SVM. Here, we take a Gaussian kernel as an example. Let $K:\mathcal{X}\times \mathcal{X}\rightarrow \mathbb{R}$ be a Gaussian kernel with $K(x_1,x_2)=\exp(-||x_1-x_2||^2_2/\sigma^2)$. Let $\tau: \mathcal{X}\rightarrow \mathbb{H}$ be a feature mapping associated with $K$ and $H=\{x\rightarrow \langle w, \tau(x)\rangle: ||w||_{\mathbb{H}}\leq \Lambda\}$, where $\langle \cdot, \cdot \rangle$ is the inner product in the reproducing kernel Hilbert space $\mathbb{H}$ and $||\cdot||_H$ is the induced norm. Suppose $\mathcal{X}\subseteq \mathbb{R}^d$ is compact and the space $\mathcal{Z}$ is equipped with the metric
$$d_{\mathcal{Z}}(z,z')=||\tau(x)-\tau(x')||_{\mathbb{H}}+\mathbbm{1}_{(y\neq y')}$$
for $z=(x,y)$ and $z'=(x',y')$. It is easy to show that $d_\mathcal{X}=||\tau(x)-\tau(x')||_{\mathbb{H}}$ is translation invariant from Gaussian kernel definition. To apply Corollary 1, we must calculate the covering numbers $\mathcal{N}(\mathcal{F},||\cdot||_\infty, \cdot)$. To this end, we embed $H$ into the space $\mathcal{C}(\mathcal{X})$ of continuous real-valued functions on $\mathcal{X}$ denoted by $I_K(H)$ equipped with the sup norm $||h||_{\mathcal{X}}:=\sup_{x\in\mathcal{X}}|h(x)|$.

We can now derive the adversarial risk bounds for the Gaussian-kernel SVM.
\begin{corollary}
For the Gaussian-kernel SVM described above, for any $f\in \mathcal{F}$, with probability at least $1-\delta$,
$$
\begin{array}{rl}
R_{P}(f, \mathcal{B})\leq& \dfrac{1}{n}\sum_{i=1}^nf(z_i)+\lambda_{f,P_n}^+\epsilon_\mathcal{B}+\dfrac{24}{\sqrt{n}}\Lambda\sqrt{d}C_3+\\
&\dfrac{30\sqrt{\pi}}{\sqrt{n}}\Lambda_{\epsilon_\mathcal{B}}+\left(1+\Lambda\right)\sqrt{\dfrac{log(1/\delta)}{2n}},
\end{array}
$$
where $\lambda_{f,P_n}^+\leq \displaystyle\max_{i}\{2y_i\langle w, \tau(x_i)\rangle, ||w||_{\mathbb{H}}\}\leq 2||w||_{\mathbb{H}}$, $C_3=(32+\dfrac{1280d}{\sigma^2})^{\frac{d+1}{2}}(2\Gamma(\frac{d+3}{2},\\
\log 2)+(\log 2)^{\frac{d+1}{2}})$, and \ $\Gamma(s,v):=\int_v^{\infty}u^{s-1}e^{-u}du$ is the incomplete gamma function.
\end{corollary}
The proof of Corollary 3 can be found in Appendix~\ref{app3}.

\textbf{Remark 6.} A margin bound for SVM in the multi-class setting can be derived in similar way. So we omit the proof.

\subsection{Neural Networks}
We next consider feed-forward neural networks. To demonstrate the generality of our method, we consider a multi-class prediction problem. We first define some notations. Let $\mathcal{Z}=\mathcal{X}\times \mathcal{Y}$, where the feature space $\mathcal{X}=\{x\in \mathbb{R}^d:||x||_2\leq  B\}$ and the label space $\mathcal{Y}=\{1,2,\cdots,k\}$; $k$ represents the number of classes. The network uses $L$ fixed nonlinear activation functions $(\sigma_1,\sigma_2,\cdots,\sigma_L)$, where $\sigma_i$ is $\rho_i$-Lipschitz and satisfies $\sigma_i(0)=0$. Given $L$ weight matrices $\mathcal{A}=(A_1, A_2, \cdots, A_L)$, the network computes the following function
$$\mathcal{H}_{\mathcal{A}}(x):=\sigma_L(A_L\sigma_{L-1}(A_{L-1}\sigma_{L-2}(\cdots\sigma_2(A_2\sigma_1(A_1x)\cdot)),$$
where $A_i\in\mathbb{R}^{d_{i}\times d_{i-1}}$ and $\mathcal{H}_{\mathcal{A}}: \mathbb{R}^d\rightarrow \mathbb{R}^k$ with $d_0=d$ and $d_L=k$. Let $W=\max\{d_0, d_1, \cdots, d_L\}$. Define a margin operator $\mathcal{M}: \mathbb{R}^k\times \{1,2,\cdots,k\}\rightarrow \mathbb{R}$ as $\mathcal{M}(v, y):=v_y-\max_{j\neq y}v_j$ and the ramp loss $l_{\gamma}: \mathbb{R}\rightarrow \mathbb{R}^+$ as 
\begin{equation*}
l_{\gamma}:=\left\{
\begin{array}{ll}
0&r<-\gamma\\
1+r/\gamma& r\in [-\gamma, 0]\\
1& r>0
\end{array}
\right..
\end{equation*}

Consider the hypothesis class $\mathcal{F}=\{(x,y)\rightarrow l_{\gamma}(-\mathcal{M}(\mathcal{H}_{\mathcal{A}}(x),y)): \mathcal{A}=(A_1, A_2, \cdots, A_L),||A_i||_{\sigma}\leq s_i, ||A_i||_{F}\leq b_i\}$, where $||\cdot||_\sigma$ represents spectral norm and $||\cdot||_{F}$ denotes the Frobenius norm. The metric in space $\mathcal{Z}$ is defined as
 $$d_{\mathcal{Z}}(z,z')=d_{\mathcal{Z}}((x,y),(x',y'))=||x-x'||_2+\mathbbm{1}_{(y\neq y')}.$$
Now we derive the adversarial expected risk for neural networks.
\begin{corollary}\label{co}
For the neural network setting defined above, for any $f\in \mathcal{F}$, with probability of $1-\delta$, the following inequality holds
\begin{equation*}
\begin{array}{l}
R_{P}(f, \mathcal{B})\leq \dfrac{1}{n}\sum_{i=1}^nf(z_i)+\lambda_{f,P_n}^+\epsilon_\mathcal{B}+\sqrt{\dfrac{log(1/\delta)}{2n}}+\\
\ \ \ \ \ \ \ \ \ \ \ \ \ \ \ \ \ \ \ \ \dfrac{288}{\gamma\sqrt{n}}\prod_{i=1}^L\rho_is_iBW \left(\sum_{i=1}^L\left(\dfrac{b_i}{s_i}\right)^{1/2}\right)^2+\\
\ \ \ \ \ \ \ \ \ \ \ \ \ \ \ \ \ \ \ \ \dfrac{12\sqrt{\pi}}{\sqrt{n}}\Lambda_{\epsilon_\mathcal{B}}\cdot (2B+1)
\end{array},
\end{equation*}
where $\lambda_{f,P_n}^+\leq \displaystyle\max_{j}\bigg\{\dfrac{2}{\gamma}\prod_{i=1}^L\rho_i||A_i||_{\sigma}, \dfrac{1}{\gamma}\big(\mathcal{M}(\mathcal{H}_{\mathcal{A}}(x_j),y_j)+\max{\mathcal{H}_{\mathcal{A}}(x_j)}-\min{\mathcal{H}_{\mathcal{A}}(x_j)}\big)\bigg\}$.
\end{corollary}
The proof of this Corollary is provided in Appendix~\ref{app4}.

\noindent
\textbf{Remark 7.} Setting $\epsilon_\mathcal{B}=0$, we can obtain a risk bound for neural networks in terms of the spectral norm and the Frobenius norm of the weight matrices; see (\ref{eq9}). Although the result (\ref{eq9}) is similar to the results in \citet{bartlett2017spectrally} and \citet{neyshabur2017pac}, the proof technique is different. We hope that our approach provides a different perspective on the generalization analysis of deep neural networks.
\begin{equation}\label{eq9}
\begin{array}{rl}
R_{P}(f)\leq& \dfrac{1}{n}\sum_{i=1}^nf(z_i)+\sqrt{\dfrac{log(1/\delta)}{2n}}+\\
&\dfrac{288}{\gamma\sqrt{n}}\prod_{i=1}^L\rho_is_iBW \left(\sum_{i=1}^L\left(\dfrac{b_i}{s_i}\right)^{1/2}\right)^2.\\
\end{array}
\end{equation}
\subsection{Principal Component Analysis}
Until now we consider the adversarial learning problem in a supervised learning setting. In this example, we show that our approach could be easily extended to unsupervised learning setting, namely the principal component analysis. We formalize PCA as follows. Fix $k\in[1,m]$ and let $\mathcal{Z}\subset \mathbb{R}^m$ such that $\max_{z\in\mathcal{Z}}||z||_2\leq B$. Define $\mathcal{T}^k$ as the set of $m$-dimensional rank-$k$ orthogonal projection matrices. PCA consists of projecting the $m$-dimensional input data onto a $k$-dimensional linear subspace which minimizes reconstruction error, and the hypothesis space for PCA is defined as $\mathcal{F}=\{z\rightarrow ||Tz-z||_2^2: T\in \mathcal{T}^k\}$. Note that the definition for adversarial expected and empirical risk in supervised learning setting given in Section~\ref{sec3} applies automatically to PCA, and the adversarial expected risk bound for PCA is given as follows. 
\begin{corollary}
For PCA we define above, for any $f\in \mathcal{F}$, with probability of $1-\delta$, the following inequality holds
\begin{equation*}
\begin{array}{rl}
R_{P}(f, \mathcal{B})\leq& \dfrac{1}{n}\sum_{i=1}^nf(z_i)+\lambda_{f,P_n}^+\epsilon_\mathcal{B}+\dfrac{576B^2k\sqrt{m}}{\sqrt{n}}+\\
&\dfrac{24B\sqrt{\pi}}{\sqrt{n}}\Lambda_{\epsilon_\mathcal{B}}+B^2\sqrt{\dfrac{log(1/\delta)}{2n}}
\end{array},
\end{equation*}
\end{corollary}
where $\lambda_{f,P_n}^+\leq \displaystyle\max_{i}\{B+||Tz_i-z_i||_2\}$.

The proof of this Corollary can be founded in Appendix~\ref{app5}.

\noindent
\textbf{Remark 8.} For each example in this Section, we provide a data-dependent upper bound for $\lambda_{f,P_n}^+$. This upper bound can be used for optimizing the adversarial risk bounds, as discussed in Section~\ref{sec6}.

%


\section{Conclusions}\label{sec6}
In this paper, we propose a theoretical method for deriving adversarial risk bounds. While our method is general and can easily be applied to multi-class problems and most of the commonly used loss functions, the bound might be loose in some cases. This is mainly because we always consider the worst case so that we avoid the problem of solving the transport map. However, for some simple problems, deriving the transport map directly might provide a better bound. The other reason is that we use covering numbers instead of the expected Rademacher complexity as our upper bounds, weakening our results. It can be seen that we have an unavoidable dimension dependency. We speculate that this dependency might be avoided by replacing the covering numbers in our bounds with the Rademacher complexity of the hypothesis class, which we will address in future studies. 

In the future, one interesting problem is to develop adversarial robust algorithms based on our results. For example, our bounds suggest that minimizing the sum of empirical risk and the term $\lambda_{f,P_n}^+\epsilon_\mathcal{B}$ can help achieve adversarial robustness. However, in practice, $\lambda_{f,P_n}^+$ is usually unknown, and we only have an upper bound for $\lambda_{f,P_n}^+$. Thus, we can perform a grid search of a regularization parameter $\eta$ on the interval $[0, 1]$ and replace $\lambda_{f,P_n}^+$ with its's upper bound multiplied by $\eta$ in the objective function. Then we minimize this new objective function for each possible regularization parameter and at the end pick the solution with the minimum value of the new objective function.



\bibliography{ref}
\bibliographystyle{apalike}

\newpage
\clearpage
\appendix
\section{Proof of Lemma 3}\label{app10}
\textit{Proof of Lemma 3.} We first prove $\implies$. For any $f\in\mathcal{F}$ and any empirical distribution $z_i\sim P_n$, by Assumption 3, there exists a constant $\lambda_{f,z_i}$ such that $f(z')-f(z_i)\leq\lambda_{f,z_i} d_{\mathcal{Z}}(z_i,z')$ for any $z'\in\mathcal{Z}$, which leads to $\sup_{z'\in \mathcal{Z}}\{f(z')-\lambda_{f,z_i} d_{\mathcal{Z}}(z_i,z')-f(z_i)\})=0$. Let $\lambda^*=\max_i\{\lambda_{z_i,f}\}$. Then, for any $z_i\sim P_n$, we have $\sup_{z'\in \mathcal{Z}}\{f(z')-\lambda^* d_{\mathcal{Z}}(z_i,z')-f(z_i)\})=0$. Therefore, $\psi_{f,P_n}(\lambda^*)=0$ and the set $\{\lambda: \psi_{f,P_n}(\lambda)=0\}$ is nonempty. $\Longleftarrow$ can be directly derived from the definition.
\section{Proof of Lemma 5}\label{app1}
\textit{Proof of Lemma 5.} 
Define the $\Phi$-indexed process $X=(X_\varphi)_{\varphi\in \Phi}$ by
$$X_\varphi:=\dfrac{1}{\sqrt{n}}\sum_{i=1}^n\sigma_i\varphi(z_i).$$
Note that $\mathbb{E}[X_\varphi]=0$ for all $\varphi\in \Phi$. First we show that X is a subgaussian process with respect to the pseudometric $d_\Phi(\varphi,\varphi')$, defined as
$$d_\Phi(\varphi, \varphi'):=||f-f'||_{\infty}+(diam(Z))|\lambda-\lambda'|$$
 for  $\varphi=\varphi_{\lambda,f}$ and $\varphi'=\varphi_{\lambda',f'}$.
From the definition of $\varphi_{\lambda,f}$, it is easy to show that $||\varphi-\varphi'||_\infty\leq d_\Phi(\varphi, \varphi')$. Then for any $t\in \mathbb{R}$, we can get
\begin{equation*}
\begin{array}{l}
\mathbb{E}\left[\exp(t(X_\varphi-X_\varphi'))\right]\\
=\mathbb{E}\left[\exp(\dfrac{t}{\sqrt{n}}\sum_{i=1}^n\sigma_i(\varphi(z_i)-\varphi'(z_i)))\right]\\
=\left(\mathbb{E}\left[\exp(\dfrac{t}{\sqrt{n}}\sigma_1(\varphi(z_1)-\varphi'(z_1))\right]\right)^n\\
\leq \exp\left(\dfrac{t^2d^2_\Phi(\varphi,\varphi')}{2}\right)
\end{array},
\end{equation*}
where the second equality is by the fact that $(\sigma_i, z_i)$ are i.i.d., and the final inequality uses Hoeffding's lemma. Therefore, X is subgaussian with respect to $d_{\Phi}$. And the expected Rademacher complexity $\mathfrak{R}_n(\Phi)$ can be bounded by the Dudley entropy integral \cite{tal}:
$$\mathfrak{R}_n(\Phi)\leq \dfrac{12}{\sqrt{n}}\int_0^\infty\sqrt{log \mathcal{N}(\Phi,d_\Phi,u)}du,$$
where $ \mathcal{N}(\Phi,d_\Phi,\cdot)$ represents the covering numbers of $(\Phi,d_\Phi)$. By the definition of $d_\Phi$, it follows that 
$$\mathcal{N}(\Phi,d_\Phi,u)\leq \mathcal{N}(\mathcal{F},||\cdot||_\infty,u/2)\cdot\mathcal{N}([a,b],|\cdot|,\frac{u}{2diam(Z)})$$
and therefore
\begin{equation*}
\begin{array}{rl}
\mathfrak{R}_n(\Phi)\leq& \dfrac{12}{\sqrt{n}}\left(\displaystyle\int_0^\infty\sqrt{log\mathcal{N}(\mathcal{F},||\cdot||_\infty,u/2)}du+\right.\\
&\left.\displaystyle\int_0^\infty\sqrt{log\mathcal{N}([a,b],|\cdot|,u/(2diam(Z)))}du\right)
\end{array}.
\end{equation*}
The second integral term could be easily obtained as follows
\begin{equation*}
\begin{array}{l}
\displaystyle\int_0^\infty\sqrt{log\mathcal{N}([a,b],|\cdot|,u/(2diam(Z)))}du\\
\leq (b-a)\cdot diam(z)\displaystyle\int_0^1\sqrt{log\dfrac{1}{u}}du\\
=\dfrac{\sqrt{\pi}}{2}(b-a)\cdot diam(z)
\end{array}.
\end{equation*}
Consequently,
$$
\begin{array}{rl}
\mathfrak{R}_n(\Phi)\leq& \dfrac{12}{\sqrt{n}}\displaystyle\int_0^\infty\sqrt{log\mathcal{N}(\mathcal{F},||\cdot||_\infty,u/2)}du+\\
&\dfrac{6\sqrt{\pi}}{\sqrt{n}}(b-a) \cdot diam(Z)
\end{array}.
$$
\section{Proofs of Corollary 2 and Corollary 3}\label{app3}
\textit{Proof of Corollary 2.}
We first verify the assumption conditions in Theorem 1. Assumption 1 is evidently satisfied since $diam(\mathcal{Z})\leq (2r+1)$. For each $f\in\mathcal{F}$, assumption 2 holds with $M=1+\Lambda r$. To verify assumption 3, we can write
\begin{equation*}
\begin{array}{rl}
f(z')-f(z)&\leq \max\{0, yw\cdot x-y'w\cdot x'\}\\
&\leq \max\{0, 2yw\cdot x\mathbbm{1}_{(y\neq y')}+||w||_2||x'-x||_2\}\\
&\leq \max\{2yw\cdot x, ||w||_2\}d_{\mathcal{Z}}(z,z')
\end{array}.
\end{equation*}
So $\lambda_{f,P_n}^+\leq \displaystyle\max_{i}\{2y_iw\cdot x_i, ||w||_2\}$ and assumption 3 holds. 

To evaluate the Dudley entropy integral, we need to estimate the covering numbers $\mathcal{N}(\mathcal{F},||\cdot||_\infty,u/2)$. First observe, for any $f_1,f_2\in \mathcal{F}$, we have
$$
\begin{array}{l}
||f_1-f_2||_\infty=\sup_{x\in\mathcal{X},y\in\mathcal{Y}}|f_1(x,y)-f_2(x,y)|\\
\leq \sup_{x\in\mathcal{X},y\in\mathcal{Y}}|yw_1\cdot x-yw_2\cdot x|\leq ||w_1-w_2||_2r.
\end{array}
$$
Since $w_1, w_2$ belong to a $\Lambda$-ball in $\mathcal{R}^d$,
$$\mathcal{N}(\mathcal{F},||\cdot||_\infty,u/2)\leq \left(\dfrac{6\Lambda r}{u}\right)^d$$
for $0<u<2\Lambda r$, and $\mathcal{N}(\mathcal{F},||\cdot||_\infty,u/2)=1$ for $u\geq 2\Lambda r$, which gives
$$
\begin{array}{l}
\displaystyle\int_0^\infty\sqrt{log\mathcal{N}(\mathcal{F},||\cdot||_\infty,u/2)}du\\
\leq \displaystyle\int_0^{2\Lambda r}\sqrt{d\log(\dfrac{6\Lambda r}{u})}du\leq 6\Lambda r\sqrt{d}
\end{array}.$$
Substituting this into expression (\ref{eq2}), we get the desired result
\begin{equation*}
\begin{array}{l}
R_{P}(f, \mathcal{B})
\leq \dfrac{1}{n}\sum_{i=1}^nf(z_i)+\lambda_{f,P_n}^+\epsilon_\mathcal{B}+\dfrac{144}{\sqrt{n}}\Lambda r\sqrt{d}+\\
\ \ \ \ \ \ \ \ \ \ \ \ \ \ \ \ \ \ \ \ \dfrac{12\sqrt{\pi}}{\sqrt{n}}\Lambda_{\epsilon_\mathcal{B}}\cdot (2r+1)+(1+\Lambda r)\sqrt{\dfrac{log(\frac{1}{\delta})}{2n}}
\end{array}.
\end{equation*}
To prove Corollary 3, we need the following Proposition \cite{cuc}.
\begin{proposition}
For compact $\mathcal{X}\subseteq \mathbb{R}^d$,
$$
\begin{array}{l}
\log \mathcal{N}(I_K(H),||\cdot||_{\mathcal{X}},u)\leq dC_2\left(\log \dfrac{\Lambda}{u}\right)^{d+1}
\end{array}
$$
holds for all $0<u\leq \Lambda/2$ where $C_2=(32+\dfrac{640d(diam(\mathcal{X}))^2}{\sigma^2})^{d+1}$.
\end{proposition}
\textit{Proof of Corollary 3.} 
First note that $||\tau(x')-\tau(x)||_{\mathbb{H}}^2=2-2K(x,x')\leq 2$. So assumption 1 holds with $diam(\mathcal{Z})\leq \dfrac{5}{2}$. For any $f\in \mathcal{F}$, assumption 2 holds with $M=1+\Lambda$. For assumption 3, we have
\begin{equation*}
\begin{array}{l}
f(z')-f(z)\\
\leq \max\{0, y\langle w, \tau(x)\rangle-y'\langle w, \tau(x')\rangle\}\\
\leq \max\{0, 2y\langle w, \tau(x)\rangle\mathbbm{1}_{(y\neq y')}+|\langle w, \tau(x)-\tau(x')\rangle\}\\
\leq \max\{2y\langle w, \tau(x)\rangle, ||w||_{\mathbb{H}}\}d_{\mathcal{Z}}(z,z')
\end{array}.
\end{equation*}
So $\lambda_{f,P_n}^+\leq \displaystyle\max_{i}\{2y_i\langle w, \tau(x_i)\rangle, ||w||_{\mathbb{H}}\}$.

Now we calculate the covering numbers $\mathcal{N}(\mathcal{F},||\cdot||_\infty,u/2)$. Observe that 
$$
\begin{array}{l}
||f_1-f_2||_\infty=\sup_{x\in\mathcal{X},y\in\mathcal{Y}}|f_1(x,y)-f_2(x,y)|\\
\leq \sup_{x\in\mathcal{X}}|\langle w_1, x\rangle-\langle w_2,x\rangle|=||w_1-w_2||_{\mathcal{X}}
\end{array}.$$
So 
\begin{equation*}
\begin{array}{l}
\displaystyle\int_0^\infty\sqrt{log\mathcal{N}(\mathcal{F},||\cdot||_\infty,u/2)}du\\
\leq \displaystyle\int_0^{\infty}\sqrt{\log(\mathcal{N}(I_K(H),||\cdot||_{\mathcal{X}},\dfrac{u}{2})}du\\
=\displaystyle\int_0^{2\Lambda}\sqrt{\log(\mathcal{N}(I_K(H),||\cdot||_{\mathcal{X}},\dfrac{u}{2})}du\\
\leq \displaystyle\int_0^{\Lambda}\sqrt{\log(\mathcal{N}(I_K(H),||\cdot||_{\mathcal{X}},\dfrac{u}{2})}du+\\
\ \ \ \ \displaystyle\int_{\Lambda}^{2\Lambda}\sqrt{\log(\mathcal{N}(I_K(H),||\cdot||_{\mathcal{X}},\dfrac{\Lambda}{2})}du\\
\end{array},
\end{equation*}
where the equality uses $||w||_{\mathcal{X}}\leq \Lambda$ and the final inequality uses the monotonicity of covering numbers. Using the bound from Proposition 2, 
\begin{equation*}
\begin{array}{l}
\displaystyle\int_0^{\Lambda}\sqrt{\log(\mathcal{N}(I_K(H),||\cdot||_{\mathcal{X}},\dfrac{u}{2})}du\\
\leq \sqrt{d}\left(32+\dfrac{640d(diam(\mathcal{X}))^2}{\sigma^2}\right)^{\frac{d+1}{2}}\displaystyle\int_0^{\Lambda}(\log \dfrac{2\Lambda}{u})^{\frac{d+1}{2}}du\\
=2\Lambda \sqrt{d}\left(32+\dfrac{640d(diam(\mathcal{X}))^2}{\sigma^2}\right)^{\frac{d+1}{2}}\Gamma(\frac{d+3}{2},\log 2)\\
\end{array}.
\end{equation*}
And
$$
\begin{array}{l}
\displaystyle\int_{\Lambda}^{2\Lambda}\sqrt{\log(\mathcal{N}(I_K(H),||\cdot||_{\mathcal{X}},\dfrac{\Lambda}{2})}du\\
=\Lambda \sqrt{d}\left(32+\dfrac{640d(diam(\mathcal{X}))^2}{\sigma^2}\right)^{\frac{d+1}{2}}(\log 2)^{\frac{d+1}{2}}.
\end{array}$$
Therefore, applying Corollary 1, we get
$$
\begin{array}{ll}
R_{P}(f, \mathcal{B})\leq& \dfrac{1}{n}\sum_{i=1}^nf(z_i)+\lambda_{f,P_n}^+\epsilon_\mathcal{B}+\dfrac{24}{\sqrt{n}}\Lambda\sqrt{d}C_3+\\
&\dfrac{30\sqrt{\pi}}{\sqrt{n}}\Lambda_{\epsilon_\mathcal{B}}+(1+\Lambda)\sqrt{\dfrac{log(1/\delta)}{2n}}
\end{array}.
$$

\section{Proof of Corollary 4}\label{app4}
The goal of this section is to prove the adversarial expected risk for neural networks. To this end, it is necessary to first establish some properties of the margin operator $\mathcal{M}(v,y)=v_y-\max_{j\neq y}v_j$ and the ramp loss $l_\gamma$.
\begin{lemma}
For every $j$, $\mathcal{M}(\cdot, j)$ is 2-Lipschitz with respect to $||\cdot||_2$.
\end{lemma}
\textit{proof.}
Let $u, v$ and $y$ be given. If $\mathcal{M}(u,y)\geq\mathcal{M}(v,y)$, denote the index $j$ which satisfies that $\mathcal{M}(v,y)=v_y-v_j$. Then,
$$
\begin{array}{l}
\mathcal{M}(u,y)-\mathcal{M}(v,y)=u_y-\max_{i\neq y}u_i-v_y+v_j\\
\leq u_y-u_j-v_y+v_j\leq 2||u-v||_\infty\leq 2||u-v||_2.
\end{array}$$
Otherwise, let $j$ be the index satisfying $\mathcal{M}(u, y)=u_y-u_j$, and we obtain
$$-2||u-v||_2\leq \mathcal{M}(u,y)-\mathcal{M}(v,y).$$
Therefore, $\mathcal{M}(\cdot, j)$ is $2$-Lipschitz with respect to $||\cdot||_2$.
\begin{lemma}
For any $f\in\mathcal{F}$, we have $\lambda_{f,P_n}^+\leq C_4$ where $C_4:=\max_{j}\{\dfrac{2}{\gamma}\prod_{i=1}^L\rho_i||A_i||_{\sigma}, \dfrac{1}{\gamma}\big(\mathcal{M}(\mathcal{H}_{\mathcal{A}}(x_j),y_j)+\max{\mathcal{H}_{\mathcal{A}}(x_j)}-\min{\mathcal{H}_{\mathcal{A}}(x_j)}\big)\}$.
\end{lemma}
\begin{proof}
By the definition of $f$, for any $z$ and $z'$, we have
\begin{equation*}
\begin{array}{l}
f(z')-f(z)\\
=l_{\gamma}(-\mathcal{M}(\mathcal{H}_{\mathcal{A}}(x'),y'))-l_{\gamma}(-\mathcal{M}(\mathcal{H}_{\mathcal{A}}(x),y))\\
\leq \max\{0,\dfrac{1}{\gamma}(\mathcal{M}(\mathcal{H}_{\mathcal{A}}(x),y)-\mathcal{M}(\mathcal{H}_{\mathcal{A}}(x'),y'))\}\\
\leq \max\{ 0,\dfrac{1}{\gamma}(\mathcal{M}(\mathcal{H}_{\mathcal{A}}(x),y')-\mathcal{M}(\mathcal{H}_{\mathcal{A}}(x'),y'))+\\
\ \ \ \  \dfrac{1}{\gamma}(\mathcal{M}(\mathcal{H}_{\mathcal{A}}(x),y)-\mathcal{M}(\mathcal{H}_{\mathcal{A}}(x),y'))\}\\
\leq \max\{0,\dfrac{2}{\gamma}|\mathcal{H}_{\mathcal{A}}(x)-\mathcal{H}_{\mathcal{A}}(x')|+ \dfrac{1}{\gamma}(\mathcal{M}(\mathcal{H}_{\mathcal{A}}(x),y)-\\
\ \ \ \ \mathcal{M}(\mathcal{H}_{\mathcal{A}}(x),y')) \}\\
\leq \dfrac{2}{\gamma}\prod_{i=1}^L\rho_i||A_i||_{\sigma}||x-x'||_2+ \dfrac{1}{\gamma}\big(\mathcal{M}(\mathcal{H}_{\mathcal{A}}(x),y)+\\
\ \ \ \ \ \max{\mathcal{H}_{\mathcal{A}}(x)}-\min{\mathcal{H}_{\mathcal{A}}(x)}\big)\mathbbm{1}_{y\neq y'}\\
\leq C_4d_{\mathcal{Z}}(z,z')\\
\end{array}.
\end{equation*}
where the third inequality uses Lemma C.1. Therefore, $\lambda_{f,P_n}^+\leq C_4$.
\end{proof}

\begin{lemma}
For any two feedforward neural network $\mathcal{H}_{\mathcal{A}}$ and $\mathcal{H}_{\mathcal{A'}}$ where $\mathcal{A}=(A_1, A_2,\cdots, A_L)$ and  $\mathcal{A'}=(A_1', A_2',\cdots, A_L')$, we have the following
$$||\mathcal{H}_{\mathcal{A}}(x)-\mathcal{H}_{\mathcal{A'}}(x)||_2\leq \prod_{i=1}^L\rho_is_iB\left(\sum_{j=1}^L\dfrac{||A_i-A_i'||_\sigma}{s_i}\right).$$
\end{lemma}
\begin{proof}
We proof this by induction. Let $\Delta_i=||\mathcal{H}^i_{\mathcal{A}}(x)-\mathcal{H}^i_{\mathcal{A'}}(x)||_2$. First observe
$$
\begin{array}{l}
\Delta_1=||\sigma_1(A_1x)-\sigma_1(A_1'x)||_2\leq \rho_1||A_1x-A_1'x||_2\\
\leq \rho_1||A_1-A_1'||_\sigma||x||_2\leq \rho_1B||A_1-A_1'||_\sigma.
\end{array}$$
For any $i\geq 1$, we have the following
\begin{equation*}
\begin{array}{l}
\Delta_{i+1}\\
=||\sigma_{i+1}(A_{i+1}\sigma_{i}(A_{i}\cdots\sigma_2(A_2\sigma_1(A_1x))))-\\
\ \ \ \ \ \ \sigma_{i+1}(A_{i+1}'\sigma_{i}(A_{i}'\cdots\sigma_2(A_2'\sigma_1(A_1'x))))||_2 \\
\leq ||\sigma_{i+1}(A_{i+1}\sigma_{i}(A_{i}\cdots\sigma_2(A_2\sigma_1(A_1x))))-\\
\ \ \ \ \ \ \sigma_{i+1}(A_{i+1}'\sigma_{i}(A_{i}\cdots\sigma_2(A_2\sigma_1(A_1x))))||_2+\\
\ \ \ \ \ \ ||\sigma_{i+1}(A_{i+1}'\sigma_{i}(A_{i}\cdots\sigma_2(A_2\sigma_1(A_1x)))-\\
\ \ \ \ \ \ \sigma_{i+1}(A_{i+1}'\sigma_{i}(A_{i}'\cdots\sigma_2(A_2'\sigma_1(A_1'x))))||_2\\
\leq \rho_{i+1}||A_{i+1}-A_{i+1}'||_\sigma ||\sigma_{i}(A_{i}\cdots\sigma_2(A_2\sigma_1(A_1x)))||_2+\\
\ \ \ \ \ \ \rho_{i+1}s_{i+1}\Delta_{i}\\
\leq \rho_{i+1}||A_{i+1}-A_{i+1}'||_\sigma \prod_{j=1}^i \rho_js_jB+\rho_{i+1}s_{i+1}\Delta_{i}\\
\end{array}.
\end{equation*}
Therefore, using the induction step, we get the following 
\begin{equation*}
\begin{array}{l}
\Delta_{i+1}\\
\leq \rho_{i+1}||A_{i+1}-A_{i+1}'||_\sigma \prod_{j=1}^i \rho_js_jB+\rho_{i+1}s_{i+1}\Delta_{i}\\
\leq \rho_{i+1}||A_{i+1}-A_{i+1}'||_\sigma \prod_{j=1}^i \rho_js_jB+\\
\ \ \ \ \ \ \prod_{j=1}^{i+1}\rho_js_jB\left(\sum_{k=1}^i\dfrac{||A_k-A_k'||_\sigma}{s_k}\right)\\
=\prod_{j=1}^{i+1}\rho_js_jB\left(\sum_{k=1}^{i+1}\dfrac{||A_k-A_k'||_\sigma}{s_k}\right)
\end{array}.
\end{equation*}
\end{proof}
We now return to the proof of Corollary 4.

\noindent
\textit{Proof of Corollary 4.} First we verify the three assumptions. Assumption 1 holds with $diam(\mathcal{Z})\leq 2r+1$. Assumption 2 is self-satisfied by the definition of ramp loss with $0\leq f(z)\leq 1$. By Lemma C.2, $\lambda_{f,P_n}^+\leq C_4$. Now we proceed to upper bound the covering number for $\mathcal{F}$. For any $f$ and $f'$,
\begin{equation*}
\begin{array}{l}
||f-f'||_\infty\\=\sup_{z}|f(z)-f'(z)|\\
=\sup_{z}|l_{\gamma}(-\mathcal{M}(\mathcal{H}_{\mathcal{A}}(x),y))-l_{\gamma}(-\mathcal{M}(\mathcal{H}_{\mathcal{A'}}(x),y))| \\
\leq \sup_{x}\dfrac{2}{\gamma}||\mathcal{H}_{\mathcal{A}}(x)-\mathcal{H}_{\mathcal{A'}}(x)||_2\\
\leq \dfrac{2}{\gamma}\prod_{i=1}^L\rho_is_iB\left(\sum_{j=1}^L\dfrac{||A_j-A_j'||_\sigma}{s_i}\right)
\end{array},
\end{equation*}
where the last inequality applies lemma C.3. Since for any matrix $A$, we have $||A||_\sigma\leq ||A||_F$. The above inequality can be written as
$$||f-f'||_\infty\leq \dfrac{2}{\gamma}\prod_{i=1}^L\rho_is_iB\left(\sum_{j=1}^L\dfrac{||A_j-A_j'||_F}{s_i}\right).$$
Define $u_j, a_j$ and $\bar{a}$ as
$$u_j=\dfrac{s_jua_j}{\frac{4}{\gamma}\prod_{i=1}^L\rho_is_iB}, a_j=\dfrac{1}{\bar{a}}\left(\dfrac{b_j}{s_j}\right)^{1/2}, \bar{a}=\sum_{j=1}^L\left(\dfrac{b_j}{s_j}\right)^{1/2}.$$
So, 
$$ \dfrac{2}{\gamma}\prod_{i=1}^L\rho_is_iB\left(\sum_{j=1}^L\dfrac{u_j}{s_j}\right)=\dfrac{u}{2}.$$
Then, the covering number $\mathcal{N}(\mathcal{F},||\cdot||_\infty,u/2)$ can be bounded by
\begin{equation*}
\begin{array}{l}
\displaystyle\int_0^\infty\sqrt{\log\mathcal{N}(\mathcal{F},||\cdot||_\infty,u/2)}du\\
\displaystyle\leq \int_0^\infty\sqrt{\sum_{i=1}^L \log\mathcal{N}(A_i,||\cdot||_F,u_i)}du\\
=\displaystyle\int_0^\infty\sqrt{\sum_{i=1}^L \log\mathcal{N}(\{A_i: ||A_i||_{\sigma}\leq s_i, ||A_i||_F\leq b_i\},||\cdot||_F,u_i)}du\\
\leq \displaystyle\int_0^\infty\sqrt{\sum_{i=1}^L \log\mathcal{N}(\{A_i: ||A_i||_F\leq b_i\},||\cdot||_F,u_i)}du\\
\leq \displaystyle\int_0^\infty\sum_{i=1}^L\sqrt{ \log\mathcal{N}(\{A_i: ||A_i||_F\leq b_i\},||\cdot||_F,u_i)}du
\end{array}.
\end{equation*}
Since $A_i\in \mathbb{R}^{d_i\times d_{i-1}}$, we can regard $A_i$ as a vector in $\mathbb{R}^{m}$ with $m=d_i\cdot d_{i-1}$ and $||\cdot||_F$ as the standard Euclidean distance in $\mathbb{R}^{m}$. Then the set $\{A_i: ||A_i||_F\leq b_i\}$ forms a $b_i$-ball in $\mathbb{R}^{m}$, and the covering number for this ball could be upper bounded by
$$ \mathcal{N}(\{A_i: ||A_i||_F\leq b_i\},||\cdot||_F,u_i)\leq \left(\dfrac{3b_i}{u_i}\right)^{m}\leq \left(\dfrac{3b_i}{u_i}\right)^{W^2}$$
for $0< u_i< b_i$, and $\mathcal{N}(\{A_i: ||A_i||_F\leq b_i\},||\cdot||_F,u_i)=1$ for $u_i\geq b_i$. So, 
\begin{equation*}
\begin{array}{l}
\displaystyle\int_0^\infty\sqrt{\log\mathcal{N}(\mathcal{F},||\cdot||_\infty,u/2)}du\\
\leq \sum_{i=1}^L \bigg(\displaystyle\int_0^\infty\sqrt{ \log\mathcal{N}(\{A_i: ||A_i||_F\leq b_i\},||\cdot||_F,u_i)}\\
\ \ \ \ \ \ \ \ \ \ \ \ \ \ \ \ \ du_i \cdot\dfrac{\dfrac{4}{\gamma}\prod_{i=1}^L\rho_is_iB}{s_ia_i}\bigg) \\
\leq \sum_{i=1}^L \bigg(\displaystyle\int_0^{b_i}\sqrt{ \log\mathcal{N}(\{A_i: ||A_i||_F\leq b_i\},||\cdot||_F,u_i)}\\
\ \ \ \ \ \ \ \ \ \ \ \ \ \ \ \ \ du_i \cdot\dfrac{\dfrac{4}{\gamma}\prod_{i=1}^L\rho_is_iB}{s_ia_i}\bigg)\\
\leq \sum_{i=1}^L \dfrac{\dfrac{4}{\gamma}\prod_{i=1}^L\rho_is_iBW}{s_ia_i} \displaystyle\int_0^{b_i}\sqrt{ \log \dfrac{3b_i}{u_i}}du_i\\
= \dfrac{12}{\gamma}\prod_{i=1}^L\rho_is_iBW\sum_{i=1}^L \dfrac{b_i}{s_ia_i} \displaystyle\int_0^{\frac{1}{3}}\sqrt{ \log \dfrac{1}{u_i}}du_i\\
\leq \dfrac{12}{\gamma}\prod_{i=1}^L\rho_is_iBW \bar{a}^2\\
\end{array},
\end{equation*}
where the last inequality uses $\displaystyle\int_0^{\frac{1}{3}}\sqrt{ \log \dfrac{1}{u_i}}du_i=\dfrac{1}{6}(2\sqrt{\log 3}+3\sqrt{\pi}\textit{erfc}(\sqrt{\log 3}))<1$. Substituting it into Corollary 1, we obtain
\begin{equation*}
\begin{array}{l}
R_{P}(f, \mathcal{B})\leq \dfrac{1}{n}\sum_{i=1}^nf(z_i)+\lambda_{f,P_n}^+\epsilon_\mathcal{B}+\sqrt{\dfrac{log(1/\delta)}{2n}}+\\
\ \ \ \ \ \ \ \ \ \ \ \ \ \ \ \ \ \ \ \ \dfrac{288}{\gamma\sqrt{n}}\prod_{i=1}^L\rho_is_iBW \left(\sum_{i=1}^L\left(\dfrac{b_i}{s_i}\right)^{1/2}\right)^2+\\
\ \ \ \ \ \ \ \ \ \ \ \ \ \ \ \ \ \ \ \ \dfrac{12\sqrt{\pi}}{\sqrt{n}}\Lambda_{\epsilon_\mathcal{B}}\cdot (2B+1)
\end{array}.
\end{equation*}
\section{Proof of Corollary 5}\label{app5}
\textit{Proof of Corollary 5.}
In order to apply Corollary 1, we need to verify the three assumptions first. Assumption 1 holds with $diam(Z)=2B$. Assumption 2 follows from the observation that $f(z)=||Tz-z||_2^2\leq B^2$. And for Assumption 3, suppose $f(z)=||Tz-z||_2^2$. Then,
\begin{equation*}
\begin{array}{l}
f(z')-f(z)\\
=||Tz'-z'||_2^2-||Tz-z||^2_2\\
\leq (B+||Tz-z||_2)(||Tz'-z'||_2-||Tz-z||_2)\\
\leq (B+||Tz-z||_2)||T(z'-z)-(z'-z)||_2\\
\leq (B+||Tz-z||_2)||T-I||_\sigma||z'-z||_2\\
= (B+||Tz-z||_2)||z'-z||_2
\end{array},
\end{equation*}
where the first inequality uses $||Tz-z||_2\leq B$ for any $z\in\mathcal{Z}$, the second inequality follows from the reverse triangle inequality and the last equality holds because $||T-I||_\sigma=1$ for any $T\in\mathcal{T}^k$. Thus Assumption 3 holds with $\lambda_{f,P_n}^+\leq \displaystyle\max_{i}\{B+||Tz_i-z_i||_2\}$. The covering numbers of $\mathcal{F}$ can be calculated as follows. Observe that
\begin{equation*}
\begin{array}{l}
||f_1-f_2||_\infty\\=\sup_{z}|f_1(z)-f_2(z)|\\
=\sup_{z}\left|||T_1z-z||_2^2-||T_2z-z||_2^2\right| \\
\leq 2B^2||T_1-T_2||_\sigma\\
\leq 2B^2||T_1-T_2||_F\\
\leq 2B^2||U_1U_1^T-U_2U_2^T||_F\\
\leq 2B^2||U_1U_1^T-U_1U_2^T+U_1U_2^T-U_2U_2^T||_F\\
\leq 4B^2||U_1-U_2||_F
\end{array},
\end{equation*}
where the second inequality uses $||\cdot||_\sigma\leq ||\cdot||_F$ and the last inequality follows from the fact that $||U_1(U_1^T-U_2^T)||_F=||(U_1-U_2)U_2^T)||_F=||U_1-U_2||_F$. Then, the covering number $\mathcal{N}(\mathcal{F},||\cdot||_\infty,u/2)$ can be bounded by
\begin{equation*}
\begin{array}{l}
\displaystyle\int_0^\infty\sqrt{\log\mathcal{N}(\mathcal{F},||\cdot||_\infty,u/2)}du\\
\displaystyle\leq \int_0^\infty\sqrt{\log\mathcal{N}\left(\mathcal{T}^k,||\cdot||_F,\frac{u}{4B^2}\right)}du\\
\leq \displaystyle\int_0^\infty\sqrt{ \log\mathcal{N}(\{U\in \mathbb{R}^{m\times k}: ||U||_F\leq\sqrt{k}\},||\cdot||_F,\frac{u}{8B^2})}du
\end{array},
\end{equation*}
where the last inequality is due to the monotonicity of covering numbers. And the covering number for $\{U\in \mathbb{R}^{m\times k}: ||U||_F\leq\sqrt{k}\}$ could be bounded in the following way:
$$\mathcal{N}\bigg(\{U\in \mathbb{R}^{m\times k}: ||U||_F\leq\sqrt{k}\},||\cdot||_F,\frac{u}{8B^2}\bigg)\leq \bigg(\dfrac{24B^2\sqrt{k}}{u}\bigg)^{mk}$$
for $\dfrac{u}{8B^2}< \sqrt{k}$, and $\mathcal{N}\big(\{U\in \mathbb{R}^{m\times k}: ||U||_F\leq\sqrt{k}\},||\cdot||_F,\frac{u}{8B^2}\big)=1$ for $\dfrac{u}{8B^2}\geq \sqrt{k}$. So,
\begin{equation*}
\begin{array}{l}
\displaystyle\int_0^\infty\sqrt{\log\mathcal{N}(\mathcal{F},||\cdot||_\infty,u/2)}du\\
\displaystyle\leq \int_0^{8B^2\sqrt{k}}\sqrt{mk\log \dfrac{24B^2\sqrt{k}}{u}}du\\
\leq 24B^2\sqrt{m}k
\end{array},
\end{equation*}
where the last inequality uses $\displaystyle\int_0^{\frac{1}{3}}\sqrt{ \log \dfrac{1}{u}}du<1$. Substituting it into Corollary 1, we obtain the desired result
\begin{equation*}
\begin{array}{rl}
R_{P}(f, \mathcal{B})\leq& \dfrac{1}{n}\sum_{i=1}^nf(z_i)+\lambda_{f,P_n}^+\epsilon_\mathcal{B}+\dfrac{576B^2k\sqrt{m}}{\sqrt{n}}+\\
&\dfrac{24B\sqrt{\pi}}{\sqrt{n}}\Lambda_{\epsilon_\mathcal{B}}+B^2\sqrt{\dfrac{log(1/\delta)}{2n}}
\end{array}.
\end{equation*}
\end{document}